\documentclass[letterpaper]{article} 
\usepackage{arxiv}  
\usepackage{times}  
\usepackage{helvet}  
\usepackage{courier}  
\usepackage[hyphens]{url}  
\usepackage{graphicx} 
\usepackage{natbib}  
\usepackage{caption} 
\usepackage{algorithm}
\nocopyright

\usepackage{newfloat}
\usepackage{listings}
\DeclareCaptionStyle{ruled}{labelfont=normalfont,labelsep=colon,strut=off} 
\floatstyle{ruled}
\newfloat{listing}{tb}{lst}{}
\floatname{listing}{Listing}
\usepackage{soul}
\usepackage{makecell}
\usepackage{amsthm}
\newtheorem{theorem}{Theorem}

\usepackage{amsmath}
\usepackage{amssymb}
\newtheorem{lemma}{Lemma} 

\usepackage{graphicx}
\usepackage{caption}
\usepackage{subcaption}
\usepackage{enumitem}
\usepackage{array}
\usepackage{quotes}
\usepackage{algorithm}
\usepackage{algpseudocode}
\algdef{SE}[SUBALG]{Indent}{EndIndent}{}{\algorithmicend\ }%
\algtext*{Indent}
\algtext*{EndIndent}
\usepackage{comment}
\usepackage{xcolor} 
\definecolor{darkblue}{RGB}{0,0,150}
\usepackage[colorlinks=true,linkcolor=black,urlcolor=black,citecolor=darkblue]{hyperref}

\title{BI-DCGAN: A Theoretically Grounded Bayesian Framework for Efficient and Diverse GANs}
\author {
    Mahsa Valizadeh,  
    Rui Tuo, 
    James Caverlee
}
\affiliations {
    Texas A\&M University\\
    \{mvalizadeh, ruituo, caverlee\}@tamu.edu
}

\begin{document}

\maketitle

\begin{abstract}
Generative Adversarial Networks (GANs) are proficient at generating synthetic data but continue to suffer from mode collapse, where the generator produces a narrow range of outputs that fool the discriminator but fail to capture the full data distribution. This limitation is particularly problematic, as generative models are increasingly deployed in real-world applications that demand both diversity and uncertainty awareness. In response, we introduce BI-DCGAN, a Bayesian extension of DCGAN that incorporates model uncertainty into the generative process while maintaining computational efficiency. BI-DCGAN integrates Bayes by Backprop to learn a distribution over network weights and employs mean-field variational inference to efficiently approximate the posterior distribution during GAN training. We establishes the first theoretical proof, based on covariance matrix analysis, that Bayesian modeling enhances sample diversity in GANs. We validate this theoretical result through extensive experiments on standard generative benchmarks, demonstrating that BI-DCGAN produces more diverse and robust outputs than conventional DCGANs, while maintaining training efficiency. These findings position BI-DCGAN as a scalable and timely solution for applications where both diversity and uncertainty are critical, and where modern alternatives like diffusion models remain too resource-intensive.
\end{abstract}

\section{Introduction}

In an era where generative models are increasingly deployed in sensitive and high-stakes applications, the ability to produce diverse, uncertainty-aware outputs is more important than ever. Standard Generative Adversarial Networks (GANs) \cite{goodfellow2014generative} and even their widely used variant, Deep Convolutional GANs (DCGANs) \cite{radford2015unsupervised} remain prone to mode collapse and fail to capture model uncertainty. This significantly limits their effectiveness in real-world applications where robustness, diversity, and interpretability are essential.

Recently, diffusion models \cite{ho2020denoising, yang2023diffusion} and transformer-based generators \cite{esser2021taming} have emerged as state-of-the-art methods in generative modeling, and they achieved impressive results in the synthesizing of high-quality samples. However, these models often require substantial computational resources, extensive training time, and are generally resource-intensive~\cite{dhariwal2021diffusion, ulhaq2022efficient}. In contrast, GANs, particularly its prominant variants like DCGAN, remain attractive in domains where efficiency, representation learning, and semantic editing are important \cite{dhariwal2021diffusion}. However, GANs still face persistent challenges, including mode collapse and the absence of a principled framework for uncertainty modeling.

The core of GANs is an adversarial process where the generator $(G)$ tries to produce samples resembling the real data distribution $p_{data}$, while the discriminator $(D)$ attempts to distinguish between real and generated samples. This adversarial training is formulated as a min-max optimization problem, Equation~\ref{ganloss}.
\begin{align}
    \label{ganloss}
    \min_G \max_D V(G,D) &= \mathbb{E}_{x \sim p_{data}(x)}[\log D(x)] \nonumber\\
    &\quad+\mathbb{E}_{z \sim p_{z}(z)}[\log(1 - D(G(z)))]
\end{align}

DCGAN~\cite{radford2015unsupervised} enhances GANs by using convolutional layers instead of fully connected ones, leading to improved training stability and performance \cite{radford2015unsupervised, farajzadeh2022generative}. GANs and DCGANs have made significant contributions to AI and machine learning \cite{saxena2021generative, jabbar2021survey, motwani2020novel}, and have shown broad utility in domains such as medical image synthesis \cite{kazeminia2020gans, bushra2020survey, frid2018gan, frid2018synthetic,chuquicusma2018fool, kitchen2017deep}, forensic sketch generation \cite{bushra2021crime}, anomaly detection \cite{schlegl2017unsupervised}, and representation learning \cite{lee2018diverse, mathieu2016disentangling}. Yet, these generative networks face key training challenges \cite{wiatrak2019stabilizing, lee2020regularization, cao2018recent, Arjovsky2017TowardsPM}, particularly mode collapse, where the generator produces limited outputs that fail to reflect the full diversity of the data distribution \cite{durgadevi2021generative, thanh2020catastrophic}. Motivated by the mode collapse issue and increasing diversity of generated samples, our framework makes three key contributions: 
\begin{itemize}
    \item We integrate a weight distribution within the network using the Bayes by Backprop method and employ mean-field variational inference to approximate the posterior distributions of the weights. We refer to this network as BI-DCGAN (Bayesian-Infused Deep Convolutional Generative Adversarial Network).
    \item We provide a rigorous mathematical proof to show the enhanced diversity of BI-DCGAN, 
    based on the analysis of covariance matrices derived from generated samples. To the best of our knowledge, this is the first study to mathematically analyze and confirm the enhanced diversity of Bayesian-based DCGAN. 
    \item We validate the proposed theoretical result through extensive empirical experiments across multiple benchmark datasets.
\end{itemize}

\section{Related Work}
\label{Related Work}

A persistent challenge in training GANs is mode collapse, where the generator produces limited modes of data despite input diversity. Numerous approaches have been proposed to mitigate this through alternative objective functions and discriminator structures, such as f-GAN \cite{nowozin2016f}, least-squares GAN (LSGAN) \cite{mao2017least}, Wasserstein GAN (WGAN) \cite{arjovsky2017wasserstein}, and WGAN-GP \cite{gulrajani2017improved}. Among these, the WGAN framework employs the Wasserstein distance (Earth Mover’s Distance), offering smoother gradients and improved stability. By enforcing a 1-Lipschitz constraint on the critic, WGAN encourages the generator to produce more diverse samples. Still, even these formulations rely on point estimates of network parameters and thus overlook the epistemic uncertainty that may contribute to mode collapse during adversarial training.

To address uncertainty, several works have explored Bayesian formulations of GANs, which replace deterministic network weights with probabilistic distributions to capture model uncertainty and potentially increase sample diversity. \citet{saatci2017bayesian} introduced Bayesian GANs using stochastic gradient Hamiltonian Monte Carlo (SGHMC) to approximate posterior distributions over the parameters of both generator and discriminator. While this framework demonstrated that modeling weight uncertainty can lead to richer data representations, reliance on SGHMC introduces computational and hyperparameter tuning challenges. Similarly, \citet{chien2019variational} proposed a Bayesian framework for GAN by combining insights from GANs, Variational Autoencoders (VAEs), and Bayesian neural networks. Their method applies variational Bayesian inference to learn parameter posteriors, improving sample realism. However, their focus was primarily on performance in supervised and semi-supervised tasks rather than explicitly quantifying or proving increases in sample diversity.

In the present study, we propose the first mathematical proof that a Bayesian formulation of DCGAN, treating network weights as probability distributions, enhances sample diversity, directly addressing mode collapse. We introduce BI-DCGAN (Bayesian-Infused Deep Convolutional Generative Adversarial Network), which employs Bayes by Backprop \cite{blundell2015weight} with mean-field variational inference (MFVI) for scalable and efficient uncertainty modeling. Unlike computationally intensive samplers such as SGHMC \cite{saatci2017bayesian}, which require careful tuning of hyperparameters like momentum and step size, factors that can hinder convergence and stability, particularly in GAN settings, MFVI assumes a factorized posterior, simplifying optimization and reducing the need for such complex tuning. While \citet{blundell2015weight} focused on Bayesian inference in supervised tasks, we extend this Bayesian approach to the unsupervised generative setting, specifically targeting mode collapse. Our contributions are twofold: (1) we provide a formal mathematical proof that Bayesian treatment of weights in DCGAN increases sample diversity, effectively addressing mode collapse; and (2) we empirically validate this on standard benchmarks, showing consistent gains in diversity over conventional DCGAN baselines.

\section{BI-DCGAN Architecture}
\label{BayesDCGAN}

In this section, we present the BI-DCGAN architecture, a novel integration of Bayesian Neural Networks (BNNs) into the DCGAN framework, specifically designed to address mode collapse by enhancing sample diversity through principled uncertainty modeling. Our architecture departs from conventional DCGANs by replacing deterministic convolutional layers with Bayesian 2D convolutional and transpose convolutional layers, enabling a distribution over network weights to be learned during training. This integration builds on the Bayes by Backprop approach introduced by \citet{blundell2015weight}, adapted and extended here to suit the adversarial generative modeling context.

In this formulation, each weight in the network is treated as a random variable with a learnable posterior distribution. Specifically, we model the variational posterior $q(w|\theta)$ as a diagonal Gaussian, parameterized by mean $\mu$ and scale $\sigma = \log(1 + \exp(\rho))$, and apply the reparameterization trick:
\begin{equation}
    \centering
    \label{eq w}
    W = \mu + \log(1+exp(\rho))\odot \epsilon
\end{equation}
where $\epsilon\sim \mathcal{N}(0,I)$. The prior over weights is a scale mixture of two Gaussian densities with zero mean and distinct variances: 
\begin{equation}
    \label{prior}
    \centering
    P(w)=\prod_{j} { \pi \mathcal{N}(w_{j}\mid 0,\,\sigma^2_{1})+(1-\pi)\mathcal{N}(w_{j}\mid0,\,\sigma^2_{2}) }
\end{equation}
where \( w_{j} \) represents the jth component of the weight vector \( W \), and \( \pi \) signifies the mixture weight, controlling the influence of each Gaussian component in the prior.

The training objective 
minimizes KL divergence between the approximate posterior and the true Bayesian posterior, which decomposes as:

\begin{equation}
    \begin{split}
    \label{obj}
    \text{KL}(q(w|\theta)||p(w|D)) &= -ELBO \\
     &= \mathbb{E}_{w \sim q(w|\theta)}\left[\log q(w|\theta) \right]\\
     & \quad - \mathbb{E}_{w \sim q(w|\theta)}\left[\log p(w) \right]\\
     & \quad -\mathbb{E}_{w \sim q(w|\theta)}\left[\log p(D|w)\right]
    \end{split}
\end{equation}

Following the Law of Large Numbers~\cite{uhlig1996law}, we estimate this via Monte Carlo sampling:
{\small
\begin{equation}
    \label{objm}
    \begin{split}
    \text{KL}(q(w|\theta) || p(w|D)) &\approx \frac{1}{n} \sum_{i=1}^n \big( \log q(w^{(i)}|\theta) \\
    &\quad - \log p(w^{(i)}) - \log p(D|w^{(i)}) \Big)
    \end{split}
\end{equation}}
where \( w^{(i)} \) represents the ith Monte Carlo sample drawn from the variational posterior \( q(w|\theta) \). 
We incorporate this into GAN training by formulating Bayesian generator and discriminator losses with KL-divergence-based regularization applied to the network weights, as outlined in Equations~\ref{ganloss}~and~\ref{objm}. 
{\small
\begin{equation}
    \begin{split}
    \label{disloss}
        \text{Discriminator\_loss} &= \frac{1}{n} \sum_{i=1}^n \left( \log q(w_d^{(i)}|\theta_d) - \log p(w_d^{(i)}) \right) \\
        &\quad - \big( \log(D(x,w_d))\\
        & \quad + \log(1 - D(G(z,w_g),w_d)) \big)
    \end{split}
\end{equation}}

{\small
\begin{equation}
    \begin{split}
    \label{genloss}
        \text{Generator\_loss} &= \frac{1}{n} \sum_{i=1}^n \left( \log q(w_g^{(i)}|\theta_g) - \log p(w_g^{(i)}) \right) \\
        &\quad -  \log(D(G(z,w_g),w_d)) 
    \end{split}
\end{equation}}

The "D" and "G" in equations~\ref{disloss}~and~\ref{genloss} represent the discriminator and the generator network, respectively. Consequently, the pseudo-code for the learning process is structured around iteratively optimizing the parameters of both the discriminator and generator networks based on their respective loss functions, Algorithm~\ref{alg:Bayesian DCGAN}.

\begin{algorithm}[!t]
\caption{Learning procedure for BI-DCGAN} 
\label{alg:Bayesian DCGAN}
{\small
\begin{algorithmic}[1]
\Require m, batch size

\For{number of training iterations}
    \For{$k$ steps}
        \State Sample minibatch of $m$ noise samples $\{z^{(1)}, ..., z^{(m)}\}$ from noise distribution $p_a(z)$.
        \State Sample minibatch of $m$ examples $\{x^{(1)}, ..., x^{(m)}\}$ from data distribution $P_{\text{data}}(X)$.
        \State Sample $\epsilon $ whose entries are $i.i.d.N(0,1)$
        \State Calculate $w = \mu + \log(1 + \exp(\rho)) \odot \epsilon$
        \State Discriminator training: 
        \Indent
        \State Calculate loss as \ref{disloss} for the discriminator network
        \State Calculate the gradient of discriminator's loss with respect to $\mu_d$ and $\rho_d$:
        \Indent
        \State $\Delta\mu_d = \frac{\partial \text{loss}(w_d, \theta_d)}{\partial w_d} + \frac{\partial \text{loss}(w_d, \theta_d)}{\partial \mu_d}$
        \State $\Delta\rho_d = \frac{\partial \text{loss}(w_d, \theta_d)}{\partial w_d} \frac{\varepsilon}{1 + \exp(-\rho_d)} + \frac{\partial \text{loss}(w_d, \theta_d)}{\partial \rho_d}$
        \EndIndent
        \State Update the discriminator 
        \EndIndent
        \State Generator training: 
        \Indent
        \State Calculate loss as \ref{genloss} for the generator network
        \State Calculate the gradient of generator's loss with respect to $\mu_g$ and $\rho_g$:
        \Indent
        \State $\Delta\mu_g = \frac{\partial \text{loss}(w_g, \theta_g)}{\partial w_g} + \frac{\partial \text{loss}(w_g, \theta_g)}{\partial \mu_g}$
        \State $\Delta\rho_g = \frac{\partial \text{loss}(w_g, \theta_g)}{\partial w_g} \frac{\varepsilon}{1 + \exp(-\rho_g)} + \frac{\partial \text{loss}(w_g, \theta_g)}{\partial \rho_g}$
        \EndIndent
        \State Update the generator
        \EndIndent
    \EndFor
    
\EndFor
\end{algorithmic}}
\end{algorithm}

\section{Theoretical Proof of Diversity}
\label{proof-main}

The BI-DCGAN architecture introduced in Section~\ref{BayesDCGAN} represents a promising approach to addressing mode collapse by incorporating uncertainty through weight distributions. While the algorithmic implementation has been detailed, a critical question remains: Can we theoretically demonstrate that this Bayesian approach indeed produces more diverse samples compared to conventional DCGAN? In this section, we provide the first rigorous mathematical proof that establishes the theoretical foundation for enhanced diversity in BI-DCGAN. 

We begin by assuming that both models have been trained using their respective frameworks. We consider the generator of the conventional DCGAN, represented as a function $g_{\text{D}} $, which maps a latent space vector $v \in \mathbb{R}^{100}$ to a generated image. For BI-DCGAN, sampling a noise vector $\epsilon \sim \mathcal{N}(0, I)$  uniquely determines the generator's weights. Since the mean and variance of the weights are established during training, the BI-DCGAN generator can be expressed as a function $g_{\text{B}, \epsilon}$, explicitly capturing its dependency on $\epsilon $. To quantify the diversity of generated samples, we compute the sample covariance matrix for each generator using the following procedure:
\begin{enumerate}
    \item Sample  $N$ (a large number of) vectors $v_1,v_2,\cdots, v_N\in \mathbb{R}^{100}$ from the distribution $\mathcal{N}(0,I)$. 
    \item Sample matrix $\epsilon$ whose entries are i.i.d. $\sim\mathcal{N}(0,1)$.

    \item Feed the vectors from step 1 into the generators $g_{\text{D}} $ and $g_{\text{B}, \epsilon}$ to obtain the corresponding generated images for both conventional DCGAN and BI-DCGAN. 
    This gives the sets of images $\{g_{\text{D}}(v_1),\cdots, g_{\text{D}}(v_N)\}$ and $\{g_{\text{B},\epsilon}(v_1),\cdots, g_{\text{B},\epsilon}(v_N)\}.$
    \item Calculate the sample covariance across all components of generated images.
\end{enumerate}
The nested representation of a generator network can be expressed as:

\begin{align}
z^{(L)} &= \sigma^{(L)} \Big( W^{(L)} \ast \sigma^{(L-1)} \big( W^{(L-1)} \ast \cdots \nonumber\\
        &\quad \sigma^{(2)} \big( W^{(2)} \ast \sigma^{(1)} ( W^{(1)} \ast v \nonumber\\
        & \quad + b^{(1)} ) + b^{(2)} \big) \cdots + b^{(L-1)} \big) + b^{(L)} \Big)
\end{align}

where:
\begin{itemize}
    \item  $v$ is input tensor from latent space.
    \item $W^{(l)}$ represents convolutional filters at layer $l$.
    \item  $b^{(l)}$ is the bias terms at layer $l$.
    \item$ \sigma^{(l)}$ is the activation function at layer $l$ (such as ReLU, sigmoid, tanh).
    \item  $\ast$ denotes the convolution operation. 
    \item  $L$ is the total number of layers in the network.
\end{itemize}
We now proceed to prove the following lemmas:

\begin{lemma}
    Let $v\sim N(0,I_d)$ be an independent normal random vector and $u$ be a random matrix whose entries $u_{ij}$ are $i.i.d.$ $N(0,1)$. Then, for any arbitrary matrix of $A$, the covariance matrix of $(A\odot u)v$ is 
    \begin{align}
    \text{Cov}((A\odot u)v )=(AA^\top)\odot I
    \end{align}
    \label{cov1}

\begin{proof}
Let \( v \sim \mathcal{N}(0, I_d) \) be a standard normal random vector, and let \( u \in \mathbb{R}^{n \times d} \) be a random matrix with i.i.d.\ standard normal entries. Let \( A \in \mathbb{R}^{n \times d} \) be an arbitrary fixed matrix, and define the random vector:
\begin{align}
    x = (A \odot u) v \in \mathbb{R}^n
\end{align}
where \( \odot \) denotes the Hadamard (elementwise) product.
Since \(v\) has mean zero and is independent of \(u\), we have
   \begin{align}
     \mathrm{Cov}(x)
     \;=\;
     \mathbb{E}[\,xx^\top\,]
     \;=\;
     \mathbb{E}\bigl[\,(A \odot u)\,v\,v^\top\,(A \odot u)^\top\bigr].
   \end{align}
   By independence, \(\mathbb{E}[v\,v^\top] = I_d\), so
   \begin{align}
     \mathrm{Cov}(x) & =\mathbb{E}\bigl[(A \odot u)\,(A \odot u)^\top\bigr] \;\cdot\; \underbrace{\mathbb{E}[v\,v^\top]}_{=\,I_d}\nonumber\\
     &=\mathbb{E}\bigl[(A \odot u)\,(A \odot u)^\top\bigr].
   \end{align}
   Now, the \((i,k)\)-th entry of \((A \odot u)\,(A \odot u)^\top\) is
   \begin{align}
     \sum_{j=1}^d \bigl(A_{ij} \,u_{ij}\bigr)\,\bigl(A_{kj}\,u_{kj}\bigr).
   \end{align}
   Taking the expectation and using the fact that \(u_{ij}\) are i.i.d.\ \(\mathcal{N}(0,1)\) we have:
   \begin{itemize}
       \item If \(i \neq k\), then \(u_{ij}\) and \(u_{kj}\) are independent and each has mean 0, so
     \(\mathbb{E}[u_{ij}\,u_{kj}] = 0\).
        \item  If \(i = k\), then \(\mathbb{E}[u_{ij}^2] = 1\).
   \end{itemize}
   Therefore,
   \begin{align}   
     \mathbb{E}\Bigl[\sum_{j=1}^d (A_{ij}u_{ij})(A_{kj}u_{kj})\Bigr]
     \;=\;
     \begin{cases}
       \sum_{j=1}^d A_{ij}^2 & \text{if } i = k,\\
       0 & \text{if } i \neq k.
     \end{cases}
   \end{align}
   
   This shows that \(\mathbb{E}[(A \odot u)\,(A \odot u)^\top]\) is diagonal, with diagonal entries \(\sum_{j=1}^d A_{ij}^2\). The matrix \(AA^\top\) has entries \((AA^\top)_{i,k} = \sum_{j=1}^d A_{ij}\,A_{kj}\). Taking the Hadamard product \((AA^\top)\odot I\) sets all off-diagonal entries to 0 and keeps the diagonal entries \(\sum_{j=1}^d A_{ij}^2\). Hence,
   \begin{align}
     \mathbb{E}\bigl[(A \odot u)\,(A \odot u)^\top\bigr]
     \;=\;
     (AA^\top)\,\odot\,I.
   \end{align}
As a result,
   \begin{align}
     \mathrm{Cov}\bigl((A \odot u)\,v\bigr)
     \;=\;
     (AA^\top)\,\odot\,I,
   \end{align}
   which completes the proof.

\end{proof}
\end{lemma}

\begin{lemma}
    If $W=\mu + \log(1+\exp(\rho))\odot \epsilon$ with $\epsilon$ being a random matrix with i.i.d. $\mathcal{N}(0,1)$ entries, the covariance matrix for $Wv+b$ is calcuated as:
    \begin{align}
    \mu\mu^\top + 
    (\log(1+\exp(\rho)) \log(1+\exp(\rho))^\top) \odot I
    \end{align}
    where $v\sim\mathcal{N}(0,I_d)$ is a normal random vector.
    \label{cov2}
\end{lemma}
\begin{proof}
    \begin{align}
        \textrm{Cov}(Wv+b) & = \textrm{Cov}(Wv)\nonumber\\
        & = \textrm{Cov}( (\mu + \log(1+\exp(\rho))\odot \epsilon) v ) \nonumber\\
        & = \textrm{Cov}( (\mu v+ (\log(1+\exp(\rho))\odot \epsilon) v ) \nonumber\\
        & = \textrm{Cov}(\mu v) \nonumber \\
        & \quad + \textrm{Cov}\left((\log(1+\exp(\rho)) \odot \epsilon) v\right) \nonumber\\
        &\quad + \textrm{Cov}\left(\mu v, (\log(1+\exp(\rho)) \odot \epsilon) v\right) \nonumber\\
        & \quad + \textrm{Cov}\left( (\log(1+\exp(\rho)) \odot \epsilon) v, \mu v\right) 
     \end{align}
According to the covariance definition, for cross-covariance we have: 

    {\footnotesize
    \begin{align}
    \textrm{Cov}\big((\log(1+\exp(\rho)) \odot \epsilon) v, \mu v\big) & = \mathbb{E}[ \big((\log(1+\exp(\rho)) \odot \epsilon) v \nonumber\\
    &  - \mathbb{E}[(\log(1+\exp(\rho)) \odot \epsilon) v] \big) \nonumber\\
    &\quad \cdot (\mu v - \mathbb{E}[\mu v])^\top]
    \end{align}
    }
    Since \(\epsilon\) is zero-mean and $v_i , \epsilon _{kj}$ are independent, we have:
    \begin{align}
    \mathbb{E}[(\log(1+\exp(\rho)) \odot \epsilon) v])= \mathbf{0}
    \end{align}
    Also, as $v\sim\mathcal{N}(0,I_d)$, $\mathbb{E}[\mu v] = 0$. Therefore, the cross-covariance term will be simplified as follows: 
    {\small
    \begin{align}
    \textrm{Cov}\big((\log(1+\exp(\rho)) \odot \epsilon) v, \mu v\big)  &=  \mathbb{E}[((\log(1+\exp(\rho))  \nonumber\\
    & \qquad \odot \epsilon) v)(\mu v) ^\top]
    \end{align}}
    A similar argument applies to the covariance term $\textrm{Cov}\left( \mu v,(\log(1+\exp(\rho)) \odot \epsilon) v\right)$, so that:
    {\small
    \begin{align}
    \textrm{Cov}\left( \mu v,(\log(1+\exp(\rho)) \odot \epsilon) v\right) & = \mathbb{E}[ (\mu v)((\log(1+\exp(\rho)) \nonumber\\
    & \qquad \odot \epsilon) v )^\top]
    \end{align}
    }

    Furthermore, each entry of $((\log(1+\exp(\rho)) \odot \epsilon) v )(\mu v)^\top$ is a linear combination of $\epsilon_{ij}$ with coefficients being independent from $\epsilon$. 

    Therefore, $\mathbb{E}[ ((\log(1+\exp(\rho)) \odot \epsilon) v )(\mu v)^\top]  = \mathbf{0}$.

    \begin{align}
    &\implies \mathrm{Cov}\left(\mu v, (\log(1+\exp(\rho)) \odot \epsilon) v\right)\nonumber\\
    &= \textrm{Cov}\left( (\log(1+\exp(\rho)) \odot \epsilon) v, \mu v\right) = \mathbf{0}.
    \end{align}
   
     It is worth noting that $\textrm{Cov}(\mu v)$ is equivalent to $\textrm{Cov}(\mu v, \mu v)$, and similarly, $\textrm{Cov} ( (\log(1+\exp(\rho))\odot \epsilon) v) $ is equivalent to $ \textrm{Cov} ( (\log(1+\exp(\rho))\odot \epsilon) v,(\log(1+\exp(\rho))\odot \epsilon) v)$. For simplicity and brevity, we use this notation throughout the text. Therefore, we have:
     {\small
     \begin{align}
         \textrm{Cov}(Wv+b) & = \textrm{Cov} (\mu v)+\textrm{Cov} ( (\log(1+\exp(\rho))\odot \epsilon) v)
     \end{align}}
     
     \noindent To compute $\textrm{Cov} (\mu v)$, we have:

    \begin{align}
    \textrm{Cov}(\mu v) = \mathbb{E}[(\mu v)(\mu v)^\top] - \mathbb{E}[\mu v]\mathbb{E}[\mu v]^\top
    \end{align}
    
    \[
    v \sim \mathcal{N}(0, I_d) \implies \mathbb{E}[v] = 0 \implies \mathbb{E}[\mu v] = \mu \cdot \mathbb{E}[v] = 0
    \]
    Therefore, the second term is zero:
    \begin{align}
        \textrm{Cov}(\mu v) & = \mathbb{E}[(\mu v)(\mu v)^\top]\nonumber\\
        & = \mu \mathbb{E}[v v^\top] \mu^\top
    \end{align}
    On the other hand, we have:
    
    \begin{align}
    \mathbb{E}[v v^\top] = \textrm{Cov}(v) + \mathbb{E}[v]\mathbb{E}[v]^\top = I
    \end{align}
    By substitution:
    \begin{align}
    \textrm{Cov}(\mu v) = \mu I \mu^\top = \mu \mu^\top
    \end{align}

     In addition, by applying Lemma \ref{cov1}, we can rewrite $\textrm{Cov} ( (\log(1+\exp(\rho))\odot \epsilon) v)$ as follows: 
     \begin{align}
       \textrm{Cov} ( (\log(1+\exp(\rho))\odot \epsilon) v) &= \big( \log(1+\exp(\rho))\nonumber \\
       & \qquad \log(1+\exp(\rho))^\top \big) \odot I
     \end{align}

     Thus, the following holds, which concludes the proof:
    
     \begin{align} 
     \textrm{Cov}(Wv+b)&= \mu\mu^\top \nonumber\\
     & + (\log(1+\exp(\rho)) \log(1+\exp(\rho))^\top) \odot I
     \end{align}
     
\end{proof}
To provide an analytical understanding from the covariance matrix of the generated images in BI-DCGAN and conventional DCGAN, the following assumptions have been made:
\begin{itemize}
    \item \textbf{Assumption 1:} We ignore the activation functions.
    \item \textbf{Assumption 2:} We consider a one-layer linear model. Therefore, we can look at the output of the model as a linear function $z=Wv+b$. From this point forward, with a slight abuse of notation, we replace $\ast$ with multiplication to simplify matrix calculations.
    \item \textbf{Assumption 3-a:} We assume the mean of weights in BI-DCGAN are similar to the weights in conventional DCGAN, i.e., $\mu_i\approx w_i$. This assumption can be justified by setting $\epsilon=0$ for generating images in BI-DCGAN. This makes $\mu_i$ a local minimum for the generator of DCGAN as well. Therefore, $\mu_i$ is also an answer to our conventional DCGAN problem, and we can make a similarity assumption between $\mu_i$ and $w_i$.
    \item \textbf{Assumption 3-b: $\mu = W$}.
\end{itemize}

We now introduce Theorem \ref{thm} to show that the eigenvalues of BI-DCGAN 
are greater than or equal to those of conventional DCGAN.
\begin{theorem}
\label{thm}
    Let $\lambda_1\geq \lambda_2\geq \cdots \geq\lambda_n$ be the eigenvalues of the covariance matrix of $g_{\text{D}}(v)$ and $\mu_1\geq \mu_2\geq \cdots \geq\mu_n$ be the eigenvalues of the covariance matrix of $g_{\text{B},\epsilon}(v)$ when $v\sim\mathcal{N}(0,I_d)$ and $\epsilon$ being a random matrix with i.i.d. $\mathcal{N}(0,1)$  entries. Under the stated assumptions, we have:
    \begin{align}
    \mu_i\geq \lambda_i
    \end{align}
    for any $1\leq i \leq n$. 
\end{theorem}
\begin{proof}
    It follows from Assumption $1$ and $2$ that we can assume $g_{\text{D}}(v)=Wv+b$ and $g_{\text{B},\epsilon}(v)=W'v+b$ where $W'=\mu + \log(1+\exp(\rho))\odot \epsilon$. For $g_{\text{D}}(v)$, we have:
    
    \begin{align}
    \textrm{Cov}(Wv+b)=\textrm{Cov}(Wv)=W\textrm{Cov}(v)W^\top=WW^\top
    \end{align}
    By applying Lemma \ref{cov2} and considering Assumption $3-b$, we obtain the following for $g_{\text{B},\epsilon}(v)$:
    {\footnotesize
\begin{align}
    \textrm{Cov}(W'v+b) & = \mu\mu^\top + 
    (\log(1+\exp(\rho)) \log(1+\exp(\rho))^\top) \nonumber\\
    &\quad\odot I\nonumber\\
    & = WW^\top + 
    (\log(1+\exp(\rho)) \log(1+\exp(\rho))^\top)\nonumber\\
    &\quad\odot I
\end{align}
}
    
    Through comparing the covariance matrices obtained for $g_{\text{D}}(v)$ and $g_{\text{B},\epsilon}(v)$, we get:
    {\small
    \begin{align}
    \textrm{Cov}(W'v+b)& = \textrm{Cov}(Wv+b)\nonumber\\
    &+(\log(1+\exp(\rho)) \log(1+\exp(\rho))^\top) \odot I
    \end{align}}

    Both $\textrm{Cov}(Wv+b)$ and $(\log(1+\exp(\rho))\log(1+\exp(\rho))^\top) \odot I$ are positive semi-definite matrices. Furthermore, according to Weyl's inequality, $i$-th eigenvalue of $\textrm{Cov}(Wv+b)+(\log(1+\exp(\rho))\log(1+\exp(\rho))^\top) \odot I $ is greater than or equal to the corresponding eigenvalue of $\textrm{Cov}(Wv+b)$. This completes the proof.
\end{proof}
\begin{figure*}[ht]
    \centering
    \includegraphics[width=0.8\linewidth]{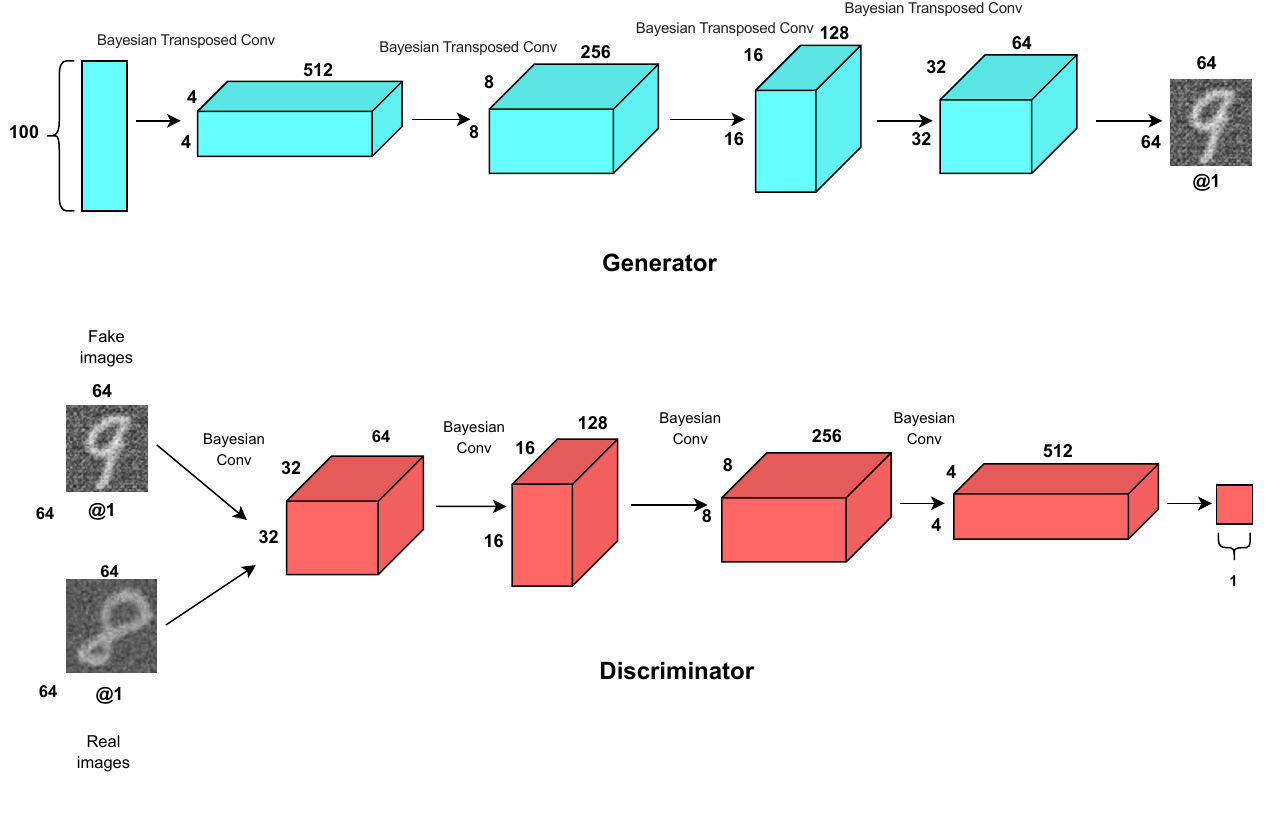} 
    \caption{BI-DCGAN architecture used in the present study.}
    \label{fig:dcbgan}
\end{figure*}

The consistently larger eigenvalues in the BI-DCGAN 
matrix indicate a greater covariance matrix, which reflects a broader spread and higher variability in the generated samples.

\section{Experiments}

Having established the theoretical superiority of BI-DCGAN in generating diverse samples through our mathematical proof, we now present empirical evidence that validates these theoretical findings. In addition, we evaluate the performance of our proposed BI-DCGAN model using various benchmark datasets. Our experiments are designed to answer the following research questions:

\begin{itemize}

    \item \textbf{RQ1:} Does the experimental analysis confirm the theoretical results about sample diversity?

    \item \textbf{RQ2:} How does the inclusion of images generated by BI-DCGAN impact the performance of a simple neural network model compared to conventional training methods?
\end{itemize}

\subsection{Dataset and Preprocessing}

We evaluated our approach using four benchmark datasets: MNIST, CIFAR-10, Fashion-MNIST, and SVHN. MNIST comprises $60,000$ training and $10,000$ test grayscale images of handwritten digits $(0-9)$, each sized $28\times28$. CIFAR-10 consists of $60,000$ color images ($32\times32$) across $10$ classes, with $50,000$ for training and $10,000$ for testing. Fashion-MNIST contains $70,000$ grayscale images of clothing items in $10$ categories, while SVHN includes over $600,000$ color images ($32\times32$) of house numbers, with $73,257$ for training and $26,032$ for testing. For consistency, all images were converted to grayscale, resized to $64\times64$, and augmented with Gaussian noise during training.

\subsection{Model Architecture and Experimental Setup}

The BI-DCGAN architecture used in this study is illustrated in Figure~\ref{fig:dcbgan}. It incorporates Bayesian 2D convolutional and transpose convolutional layers. All experiments were performed on a single NVIDIA A5000 GPU with 24GB of memory.

\begin{figure*}[t]
\centering
\begin{subfigure}[t]{0.93\columnwidth}
  \centering
  \includegraphics[width=\linewidth]{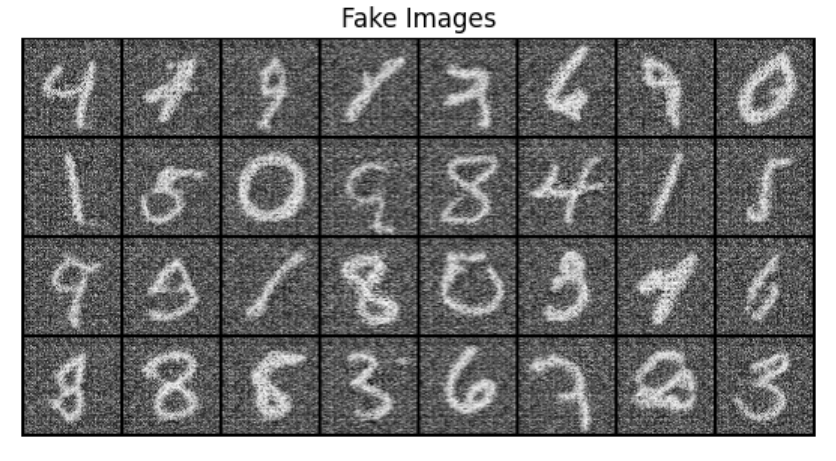}
  \caption{BI-DCGAN 
  on MNIST.}
  \label{fig:bmnist}
\end{subfigure}
\begin{subfigure}[t]{0.96\columnwidth}
  \centering
  \includegraphics[width=\linewidth]{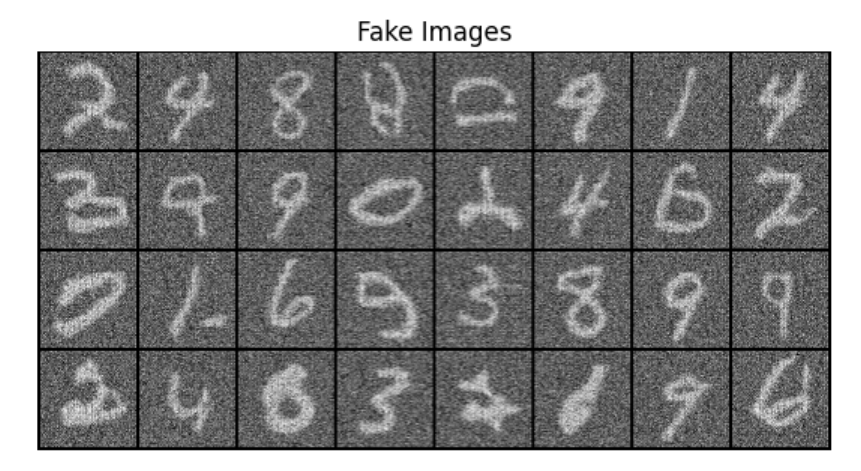}
  \caption{DCGAN on MNIST.}
  \label{fig:tmnist}
\end{subfigure}
\begin{subfigure}[t]{0.94\columnwidth}
  \centering
  \includegraphics[width=\linewidth]{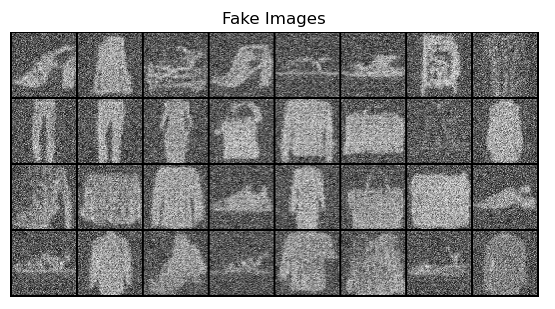}
  \caption{BI-DCGAN 
  on Fashion-MNIST.}
  \label{fig:bmnistf}
\end{subfigure}
\begin{subfigure}[t]{0.96\columnwidth}
  \centering
  \includegraphics[width=\linewidth]{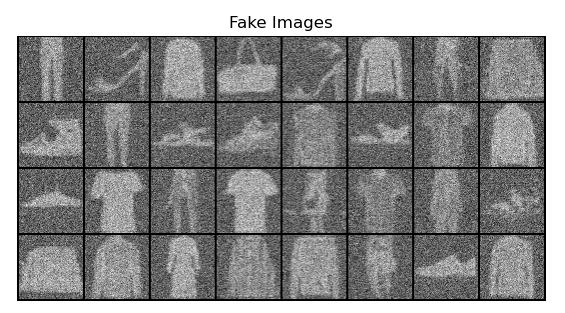}
  \caption{DCGAN on Fashion-MNIST.}
  \label{fig:tmnistf}
\end{subfigure}
\caption{Samples generated from random noise by BI-DCGAN 
and conventional DCGAN.}
\label{fig:test}
\end{figure*}

\subsection{Diversity Analysis of Generated Samples}
\label{Subdiv}

In line with our theoretical proof presented in Section ~\ref{proof-main}, we conducted a mathematical analysis to assess sample diversity. This analysis quantifies the diversity in images generated by BI-DCGAN and compares it to that of images produced by conventional DCGAN. The methodology involves the creation of two distinct datasets, each consisting of $32,000$ images produced by their respective GAN algorithms. Each image within these datasets possesses dimensions of $64\times 64$, amounting to a total of $4096$ pixels, and is treated as a vector for the purpose of this analysis. In viewing each generated image as a vector with a dimensionality of $4096$, we computed the sample covariance for each dataset, resulting in two covariance matrices of size $4096\times 4096$. This mathematical approach enables a comprehensive understanding of the relationships and variations present among the pixels of the generated images, facilitating a quantitative assessment of diversity.

Figures~\ref{fig:bmnist}~\&~\ref{fig:tmnist} showcase a selection of randomly generated MNIST digit samples produced by both the BI-DCGAN and the conventional DCGAN models, respectively. Additionally, Figures~\ref{fig:bmnistf}~\&~\ref{fig:tmnistf} display randomly generated samples from the Fashion-MNIST dataset, produced by the same models. Notably, the BI-DCGAN model exhibits the ability to generate images that closely resemble those from the real datasets. This demonstrates the model's enhanced capacity to capture and replicate the intricate patterns and features present in the original datasets, emphasizing the efficacy of the Bayesian approach in improving generative performance.

Denoting the $32,000$ images generated by conventional DCGAN as $d_1, d_2, \ldots, d_N$ and those generated by BI-DCGAN as $b_1, b_2, \ldots, b_N$ (where $N=32000$), individual components of these vectors are represented as follows:

\[d_i=(d_i^1,\cdots, d_i^{4096})\]
\[b_i=(b_i^1,\cdots, b_i^{4096})\]

The population (or sample) mean is defined as:
\[\bar{d}=(\bar{d}^1,\cdots, \bar{d}^{4096})\]
where for any $1\leq j\leq 4096$: 
\begin{align}  
\bar{d}^j = \frac{1}{N}\left(\sum_{i=1}^N d_i^j\right)
\end{align}

This definition holds for the images generated by BI-DCGAN as well. Further, for any $1\leq r,s\leq 4096$, the population covariance between $r^{th}$ and $s^{th}$ components is defined as:
\begin{align}
Cov(d^r, d^s) = \frac{1}{N}\left(\sum_{i=1}^N (d_i^r - \bar{d}^r)(d_i^s - \bar{d}^s)\right) 
\end{align}

This computation results in the $(r,s)$ entry of the population covariance matrix for the data generated by conventional DCGAN, and a similar definition can be applied for BI-DCGAN. According to the theoretical proof of diversity (Section~\ref{proof-main}) the diversity criterion is the eigenvalues of the covariance matrix. Let 
\[\lambda_1 \geq \lambda_2 \geq \cdots \geq \lambda_{4096}\]
denote the sorted eigenvalues of the population covariance matrix for the conventional DCGAN, and
\[\mu_1 \geq \mu_2 \geq \cdots \geq \mu_{4096}\] 
represent the corresponding eigenvalues for the BI-DCGAN. From our experiment, it was observed that:
\begin{align}
    \lambda_1 < \mu_1, \quad \lambda_2 < \mu_2, \quad \dots, \quad \lambda_{4096} < \mu_{4096}.
\end{align}

Table~\ref{tab:eigenvalues} shows the first seven eigenvalues (\(\mu_i\) and \(\lambda_i\)) for BI-DCGAN and conventional DCGAN, respectively, across all four datasets. The observation that the claim holds for the first seven terms, and, in fact, for all indices $j$ (i.e., $\mu_j > \lambda_j$), provides compelling evidence supporting the assertion that BI-DCGAN exhibits greater diversity when compared to conventional DCGAN. 
Therefore, the mathematical analysis provides strong support for our theoretical proof, 
confirming that the BI-DCGAN model generates a more diverse set of samples compared to the conventional DCGAN.

\begin{table}[!t]

\centering
\begin{tabular}{|>{\centering\arraybackslash}m{2.3cm}|>{\centering\arraybackslash}m{2.2cm}|>{\centering\arraybackslash}m{2.2cm}|}

\hline
Dataset & BI-DCGAN  \newline $\mu$ & Conventional DCGAN \newline $\lambda$ \\

\hline
& 26.298 & 11.399 \\
&18.938 & 9.089 \\
&16.208 & 7.628 \\
\textbf{MNIST}&14.185 & 6.457 \\
&11.764& 5.500 \\
&11.603 & 4.990 \\
&8.771&3.967\\
\hline

& 16.581 & 11.863 \\
&7.147 & 5.039 \\
&4.606 & 2.909 \\
\textbf{CIFAR-10 }&2.168 & 1.601 \\
&2.065& 1.403 \\
&1.889 & 1.249 \\
&1.730 & 1.009 \\
\hline

& 48.975 & 20.173 \\
&33.545 & 13.182 \\
&13.260 & 4.939 \\
\textbf{Fashion-}&10.131 & 4.422 \\
\textbf{MNIST}&7.629& 3.295 \\
&7.161 & 2.872 \\
&5.436 & 2.328 \\
\hline

& 50.239 & 36.729 \\
&6.234 & 3.943 \\
&5.836 & 3.279\\
\textbf{SVHN}&4.955 & 2.468 \\
&2.342& 1.312 \\
&1.890 & 1.141 \\
&1.722 & 0.990 \\
\hline

\end{tabular}
\caption{Comparing Eigenvalues Across Different Datasets.}
\label{tab:eigenvalues}
\end{table}

\subsection{Impact on Model Performance}

Furthermore, to validate our findings, we compared two models: one trained solely on a $10\%$ subset of the MNIST dataset and another trained on a combination of a $10\%$ subset of the MNIST dataset and approximately $1{,}400$ images generated by BI-DCGAN for the MNIST dataset. The generated images underwent a rigorous labeling process to ensure quality and accuracy. Each image was independently labeled by five individuals, and only those for which all five annotators unanimously agreed on the label were included in the dataset. This stringent criterion was implemented to eliminate ambiguity and maintain high-quality training data. Both models were evaluated on unseen images from BI-DCGAN, and the results are presented in Table~\ref{tab:general}. Interestingly, even when trained on a small amount of generated data, the model exhibited better performance and greater generalization. 

On the other hand, in the pursuit of reducing uncertainty in neural networks, deep ensembles have gained significant attention. This method involves training neural networks from scratch multiple times on the dataset, resulting in distinct models. During testing, the dataset is passed through each model, and the final output is obtained by averaging their outputs. Ensemble strategies vary in their approaches to selecting baseline classifiers for training, with two primary types: homogeneous and heterogeneous ensembles. While homogeneous ensembles utilize the same type of baseline classifiers trained on different data subsets, heterogeneous ensembles employ various types of classifiers trained on the same dataset \cite{mohammed2023comprehensive}. In our study, we chose heterogeneous ensembles as the baseline method for uncertainty reduction (Table \ref{tab:general}) to compare the performance of the model trained on a combination of MNIST and the generated images with that of a model trained using heterogeneous ensemble learning. Notably, the model trained on the combined dataset outperformed the model trained through ensemble learning. It is worth mentioning that we employed a simple model with only two convolutional layers and trained it for a few epochs to facilitate result comparison.

\begin{table}[!t]
\centering
\begin{tabular}{|>{\centering\arraybackslash}m{2.2cm}|>{\centering\arraybackslash}m{2.2cm}|>{\centering\arraybackslash}m{2.2cm}|}
\hline
Trained on MNIST & Ensemble learning \newline on MNIST & Trained on MNIST \newline and generated images \\
\hline
82\% & 83.6\% & 86\%\\
\hline
\end{tabular}
\caption{Comparison of different models.}
\label{tab:general}
\end{table}

\section{Conclusion}

This study addressed mode collapse in DCGAN by introducing BI-DCGAN, a Bayesian framework that incorporates uncertainty modeling via Bayes by Backprop and mean-field variational inference. A central contribution of our work is the first rigorous mathematical proof demonstrating that BI-DCGAN generates more diverse samples than conventional DCGANs, based on a covariance matrix analysis. This theoretical finding is further validated by empirical evaluations across benchmark datasets, which reveal larger eigenvalues in the generated sample covariances, indicative of a broader spread and increased variability. These outcomes strongly support our theoretical result and highlight the effectiveness of the Bayesian approach in addressing mode collapse in generative models.

\bibliography{custom}

\appendix
\section{Appendix}

\subsection{Bayesian Neural Network}

Deep learning has gained significant attention across various fields but suffers from overfitting \cite{szegedy2013intriguing} and overconfident estimates \cite{goan2020bayesian}, necessitating careful management.
A Bayesian Neural Network (BNN) deviates from the conventional deep neural network paradigm, offering a distinctive solution to the challenges posed by overfitting in deep learning models. In contrast to conventional neural networks that provide point estimates for weights, a BNN incorporates uncertainty into its predictions by representing the weights and biases as probability distributions instead of fixed values (Figure~\ref{fig:bnn}). This distinctive approach enables the incorporation of prior knowledge regarding these parameters into the model, allowing for the continuous refinement of beliefs as new data becomes available. By embracing probabilistic representations, Bayesian Neural Networks provide a flexible and adaptive framework, particularly useful for scenarios where uncertainty estimation is crucial. This design allows the model to not only make predictions but also express the range of potential outcomes, offering a more comprehensive understanding of the data and facilitating informed decision-making.

\begin{figure}[h]
    \centering
    \includegraphics[scale=0.40]{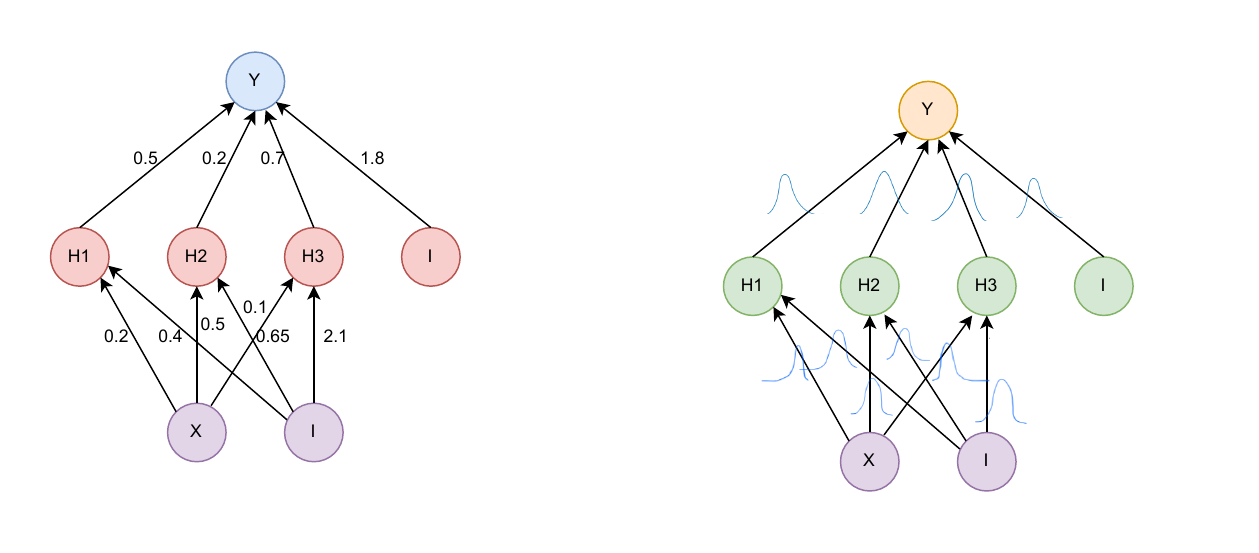}
    \caption{Bayesian Neural Networks (right) vs. conventional Neural Networks.}
    \label{fig:bnn}
\end{figure}

BNNs represent a unique class of stochastic neural networks that integrate Bayesian inference principles into their architecture. In the context of BNNs, network parameters, encompassing weights and biases, are treated as random variables, and a probability distribution is systematically defined over them \cite{goan2020bayesian}. The distinguishing feature of BNNs lies in their capacity to learn and adapt to the distribution of weights. This learning process empowers BNNs to provide not just point estimates but a comprehensive measure of uncertainty regarding their predictions. This aspect of quantifiable uncertainty proves to be highly advantageous across a diverse array of applications.

Bayes' Theorem offers a powerful tool for representing a distribution over parameters conditioned on the observed data, denoted as the posterior probability distribution $P(W|D)$. The theorem is expressed through the equation:
{\small
\begin{equation}
    \label{eq}
    \centering
    P(W|D) = \frac{P(D|W) P(W)}{P(D)} = \frac{P(D|W) P(W)}{ \int_{W^\prime}  P(D|W^\prime ) P(W^\prime ) \,dW^\prime }
\end{equation}}

In this context, $P(W)$ represents the prior probability distribution, $P(D|W)$ denotes the data likelihood, and $P(D)$ is the marginal probability (also refers as the normalization constant). This method of computing the posterior $P(W|D)$ is commonly referred to as exact inference. However, as indicated in the equation~\ref{eq}, computing $P(D)$ requires an integral over the weight space, which is often intractable. To overcome this challenge, two primary approaches have been proposed: Markov Chain Monte Carlo (MCMC) and Variational Inference. While Saatci \textit{et al.} \cite{saatci2017bayesian} utilizes Hamiltonian Monte Carlo (an MCMC method), Blundell \textit{et al.} \cite{blundell2015weight} employs the variational inference method to estimate the posterior distribution, a technique also employed in the current study.

\subsection{Mean Field Variation Inference Method for BNN}
\label{mfvibnn}
Variational inference is a method that seeks to approximate the posterior distribution $P(W|D)$ by introducing a surrogate distribution denoted as $q$. This surrogate, known as the variational distribution, is selected from a tractable family of distributions. The objective is to find the distribution $q$ that is the ``closest'' to the true posterior, thereby facilitating more manageable computations. In MFVI, the algorithm operates under the assumption that the variational family is fully factorized. This means that the joint posterior distribution, denoted as \(p(w | D)\), can be approximated by the product of individual variational distributions for each latent random variable. In other words, the approximation is expressed as:

\begin{equation}
    \label{mfvi}
    \centering
    p(w | D) \approx \prod_{i=1}^{n} q_i(w_i)
\end{equation}

As stated previously, variational inference aims to identify the distribution \(q(w)\) that best approximates the true posterior. The proximity between the two distributions is quantified using the Kullback-Leibler (KL) divergence \cite{kullback1951information, joyce2011kullback}, denoted as \(KL(q||p)\), defined as:

\begin{align}
    \label{kl}
    \text{KL}(q(w)||p(w|D)) &= \int_{w} q(w) \log\frac{q(w)}{p(w|D)} \,dw \\
    &= \mathbb{E}\left[\log\frac{q(w)}{p(w|D)}\right]
\end{align}

where \(q(w)\) is the variational distribution, and \(p(w|D)\) is the true posterior distribution given observed data \(x\). Also referred to as information gain, KL divergence measures the "information lost when \(p(w|D)\) is approximated by \(q(w)\)". If the true posterior distribution \(p(w|D)\) and the variational distribution \(q(w)\) are identical, the KL divergence \(KL(q(w)||p(w|D)))\) equals zero, indicating that no information is lost in the approximation. On the contrary, as \(p\) and \(q\) diverge, the value of \(KL(q(w)||p(w|D))\) increases, signifying the growing difficulty in predicting the true distribution \(p(w|D)\) based on the approximation \(q(w)\). Minimizing the KL divergence corresponds to making the variational distribution \(q(w)\) as close as possible to the true posterior \(p(w|D)\). By expanding equation~\ref{kl}, we have:
{\small
\begin{equation}
    \label{elbo}
    \begin{split}
        \text{KL}(q(w)||p(w|D)) &= \int_{w} q(w) \log\frac{q(w)}{p(w|D)} \,dw \\
        &= \int_{w} q(w) \log\frac{q(w)p(w)}{p(w,D)} \,dw\\
        &= \int_{w} q(w) \log\frac{q(w)}{p(w,D)} \,dw \\
        &\quad+ \int_{w} q(w) \log p(D) \,dw \\
        &= \int_{w} q(w) \log\frac{q(w)}{p(D|w)p(w)} \,dw\\  &\quad+ \log p(D)\\
        &= \int_{w} q(w) \log\frac{q(w)}{p(w)} \,dw \\
        &\quad- \int_{w} q(w) \log p(D|w) \,dw+ \log p(D)
    \end{split}
\end{equation}}

In the above equation, $\int_{w} q(w) \log\frac{q(w)}{p(w)} \,dw  - \int_{w} q(w) \log p(D|w) \,dw$ is negative of evidence lower bound (ELBO) \cite{kingma2019introduction}, so that:

{\small
\begin{equation}
    \label{elbo2}
    \begin{split}
        ELBO &= \int_{w} q(w) \log p(D|w) \,dw - \int_{w} q(w) \log\frac{q(w)}{p(w)} \,dw 
    \end{split}
\end{equation}}

Considering equation~\ref{kl}, ELBO can be rewritten as follows:

\begin{equation}
    \label{elbo3}
    ELBO = \mathbb{E}_{w \sim q(w)}\left[\log p(D|w)\right] - \text{KL}(q(w)||p(w))
\end{equation}

Therefore, minimizing \(KL(q(w)||p(w|D))\) is equivalent to minimizing $-ELBO$, or maximizing $ELBO.$

\subsection{BI-DCGAN Architecture}
\label{BayesDCGANs}

The Generative Adversarial Network utilized in the current study is specifically the DCGAN. DCGAN is a variant of GAN that exhibits a distinctive architecture where both the discriminator and generator incorporate convolutional layers. The use of convolutional layers is particularly advantageous as it allows the network to effectively capture spatial dependencies within the data. This spatial awareness significantly contributes to enhancing the overall quality of the generated images.
In the endeavor to incorporate Bayesian Neural Networks (BNNs) into the framework of DCGAN, the methodology presented by Blundel \textit{et al.} \cite{blundell2015weight} was followed. This approach involves treating the weights within the neural network as random variables and introducing a distribution over these weights. The motivation behind adopting this Bayesian perspective is to enable the model to account for uncertainty in its predictions, contributing to a more robust and versatile generative model. 

A main aspect highlighted in Blundell \textit{et al.} \cite{blundell2015weight}  involves the utilization of the reparameterization technique. This technique plays a crucial role in ensuring that the variational parameters are sampled from a specific distribution, introducing a layer of stochasticity to the model. The assumption made in the study is that the variational posterior follows a diagonal Gaussian distribution. The process initiates by generating a sample from a unit Gaussian distribution, labeled as $\epsilon$. This sample undergoes a deterministic transformation, involving a shift by a mean $\mu$ and a scaling by a standard deviation $\sigma$, as specified in Equation~\ref{eq_w}. Consequently, this process yields the weight parameters $W$.

\begin{equation}
    \centering
    \label{eq_w}
    W = \mu + \log(1+exp(\rho))\odot \epsilon
\end{equation}

Here, $\mu$ represents the mean of the distribution, and $\rho$ is a parameter used to determine the standard deviation through the transformation $\sigma = \log(1 + \exp(\rho))$. Taking into account the shift and scaling operations, the weight parameters $W$ follow a normal distribution $\mathcal{N}(\mu, \sigma^2)$, where the set of learnable variational posterior parameters is represented as $\theta = (\mu, \rho)$. Thus, the variational posterior can be denoted as $q(w|\theta)$.

As previously mentioned, we adopt (similar to Blundell \textit{et al.} \cite{blundell2015weight}) a Gaussian variational posterior. However, in terms of the prior, a scale mixture of two Gaussian densities with zero mean and distinct variances, where \( \sigma_{1} > \sigma_{2} \) and $\sigma_2 \ll 1$, is considered. The prior distribution is expressed as follows:

\begin{equation}
    \label{priors}
    \centering
    P(w)=\prod_{j} { \pi \mathcal{N}(w_{j}\mid 0,\,\sigma^2_{1})+(1-\pi)\mathcal{N}(w_{j}\mid0,\,\sigma^2_{2}) }
\end{equation}

Here, \( w_{j} \) represents the jth component of the weight vector \( W \), and \( \pi \) signifies the mixture weight, controlling the influence of each Gaussian component in the prior. 

The objective function of Bayesian Neural Networks is to minimize KL divergence, according to Blundell \textit{et al.} \cite{blundell2015weight} and the Equation~\ref{elbo3}, which can be rewritten as follows:

\begin{equation}
    \begin{split}
    \label{objs}
    \text{KL}(q(w|\theta)||p(w|D)) &= -ELBO \\
     &= \mathbb{E}_{w \sim q(w|\theta)}\left[\log q(w|\theta) \right]\\
     & \quad - \mathbb{E}_{w \sim q(w|\theta)}\left[\log p(w) \right]\\
     & \quad -\mathbb{E}_{w \sim q(w|\theta)}\left[\log p(D|w)\right]
    \end{split}
\end{equation}

According to the Law of Large Numbers (LLN) \cite{uhlig1996law}, a fundamental statistical principle,   \( \lim_{{n \to \infty}} \frac{1}{n} \sum_{{i=1}}^n x_i = E(x) \) with a probability of 1. This law ensures that as the size of a sample increases, the sample mean converges towards the expected mean of the entire population. Utilizing LLN and drawing a sufficient number of Monte Carlo samples from the variational posterior \( q(w|\theta) \), we can rewrite Equation~\ref{objs} as follows:

{\small
\begin{equation}
    \label{objms}
    \begin{split}
    \text{KL}(q(w|\theta) || p(w|D)) &\approx \frac{1}{n} \sum_{i=1}^n \big( \log q(w^{(i)}|\theta) \\
    &\quad - \log p(w^{(i)}) - \log p(D|w^{(i)}) \Big)
    \end{split}
\end{equation}}

where \( w^{(i)} \) represents the ith Monte Carlo sample drawn from the variational posterior \( q(w|\theta) \).

In the present study, we build upon the foundational concepts of Bayesian neural networks introduced by Blundell \textit{et al.} \cite{blundell2015weight} by integrating them into a Convolutional 2-dimensional (Conv2D) and Convolutional 2-dimensional transpose (Conv2D Transpose) architecture to develop our BI-DCGAN model. This innovative approach allows us to leverage the strengths of convolutional layers in processing image data while incorporating the Bayesian framework to enhance diversity of the generated samples. 

Therefore, the formulation of loss function for our model involves the consideration of two key loss functions outlined in Equations~\ref{ganloss}~and~\ref{objms}. These loss functions encapsulate the essential components for training the BI-DCGAN model. The discriminator loss, as expressed in Equation~\ref{disloss1}, is calculated by evaluating the difference between two main components. The first component entails a probabilistic comparison between the variational posterior and prior of the discriminator network. The second component is related to the adversarial nature of GANs and consists of the binary cross-entropy between the log probability of the real data \(D(x)\) and the log probability of the generated data \(D(G(z))\), Equation~\ref{disloss1}.

Likewise, the generator loss, as described in Equation~\ref{genloss1}, is determined by assessing the probabilistic distinction between the variational posterior and the prior distribution of the generator network. The second component involves the binary cross-entropy between the log probability of the generated data \(D(G(z))\).

{\small
\begin{equation}
    \begin{split}
    \label{disloss1}
        \text{Discriminator\_loss} &= \frac{1}{n} \sum_{i=1}^n \left( \log q(w_d^{(i)}|\theta_d) - \log p(w_d^{(i)}) \right) \\
        &\quad - \big( \log(D(x,w_d))\\
        & \quad + \log(1 - D(G(z,w_g),w_d)) \big)
    \end{split}
\end{equation}

\begin{equation}
    \begin{split}
    \label{genloss1}
        \text{Generator\_loss} &= \frac{1}{n} \sum_{i=1}^n \left( \log q(w_g^{(i)}|\theta_g) - \log p(w_g^{(i)}) \right) \\
        &\quad -  \log(D(G(z,w_g),w_d)) 
    \end{split}
\end{equation}}

The "D" and "G" in equations~\ref{disloss1}~and~\ref{genloss1} represent the discriminator and the generator network, respectively.

\begin{figure*}[htbp]
\centering
\begin{subfigure}{0.45\textwidth}
  \centering
  \includegraphics[width=\linewidth]{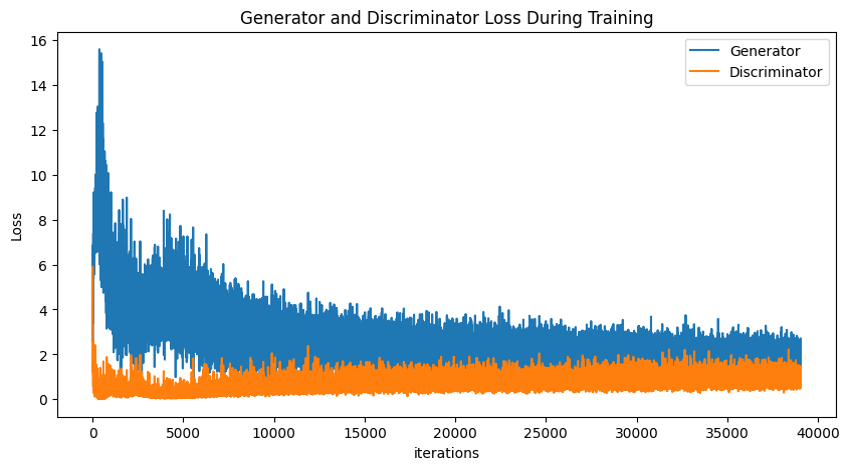}
  \caption{BI-DCGAN loss-MNIST.}
  \label{fig:bayesdcgan}
\end{subfigure}
\begin{subfigure}{0.45\textwidth}
  \centering
  \includegraphics[width=\linewidth]{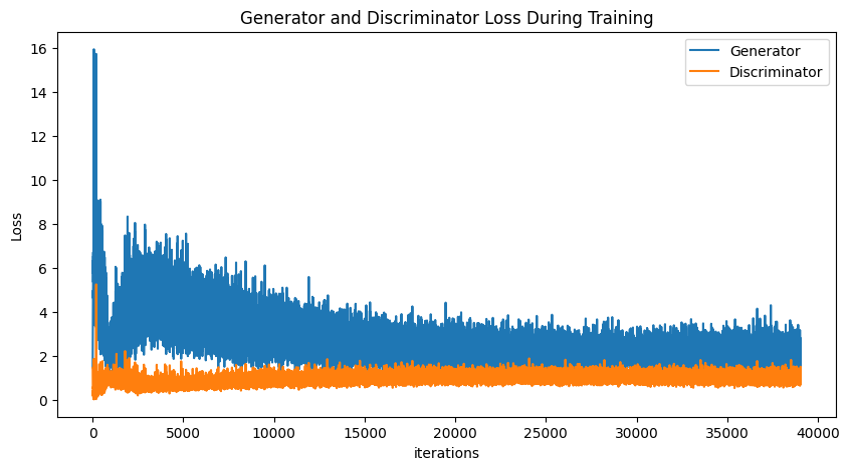}
  \caption{Conventional DCGAN loss-MNIST.}
  \label{fig:tdcgan}
\end{subfigure}

\begin{subfigure}{0.45\textwidth}
  \centering
  \includegraphics[width=\linewidth]{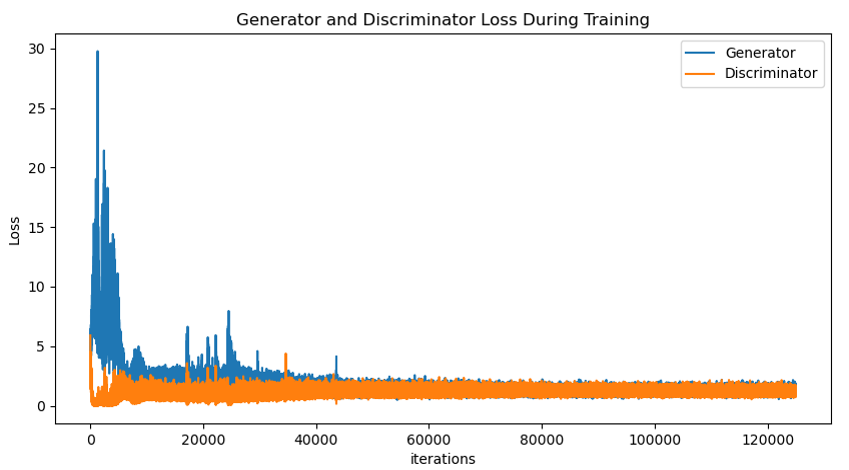}
  \caption{BI-DCGAN loss-CIFAR10.}
  \label{fig:bdcganc}
\end{subfigure}
\begin{subfigure}{0.45\textwidth}
  \centering
  \includegraphics[width=\linewidth]{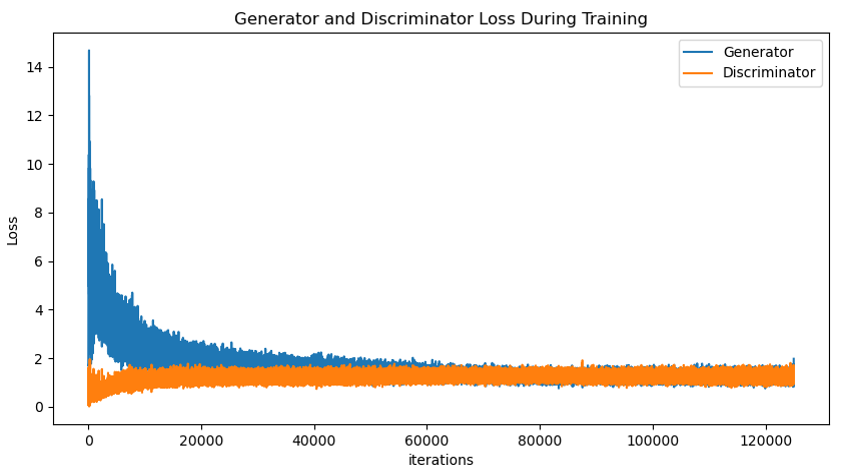}
  \caption{Conventional DCGAN loss-CIFAR10.}
  \label{fig:tdcganc}
\end{subfigure}
\begin{subfigure}{0.45\textwidth}
  \centering
  \includegraphics[width=\linewidth]{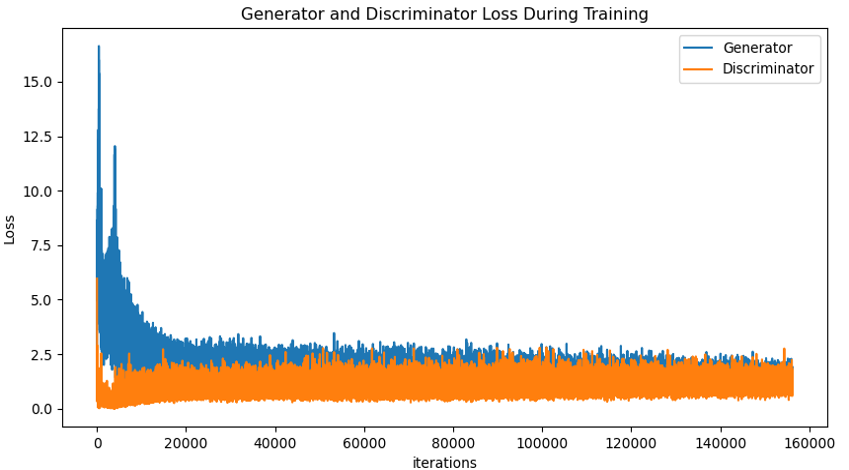}
  \caption{BI-DCGAN loss-FashionMNIST.}
  \label{fig:bdcganf}
\end{subfigure}
\begin{subfigure}{0.45\textwidth}
  \centering
  \includegraphics[width=\linewidth]{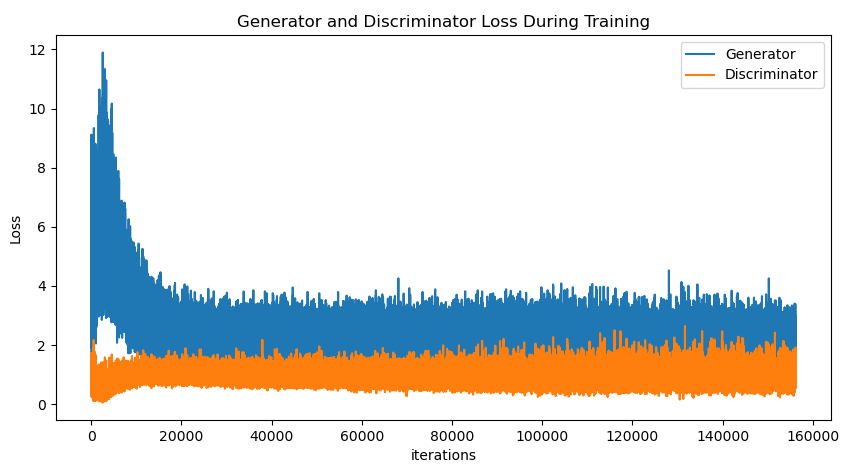}
  \caption{Conventional DCGAN loss-FashionMNIST.}
  \label{fig:tdcganf}
\end{subfigure}

\begin{subfigure}{0.45\textwidth}
  \centering
  \includegraphics[width=\linewidth]{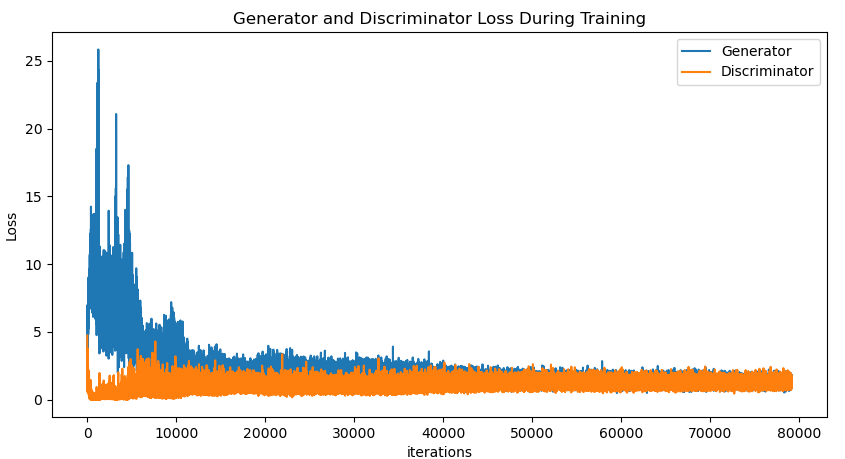}
  \caption{BI-DCGAN 
  loss-SVHN.}
  \label{fig:bdcgans}
\end{subfigure}
\begin{subfigure}{0.45\textwidth}
  \centering
  \includegraphics[width=\linewidth]{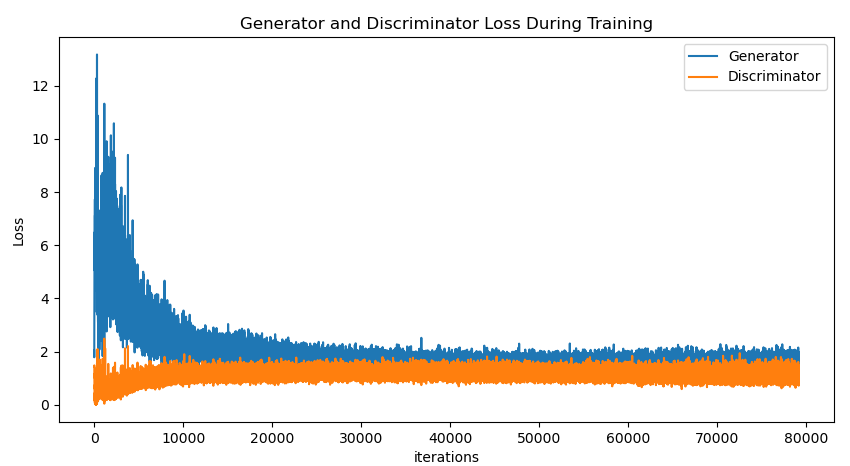}
  \caption{Conventional DCGAN loss-SVHN.}
  \label{fig:tdcgans}
\end{subfigure}

\caption{The generator and discriminator loss during training for BI-DCGAN and conventional DCGAN.}
\vspace{11pt}
\label{fig:losses}
\end{figure*}

\subsection{Dataset and Preprocessing}

The MNIST dataset, a benchmark in the field of machine learning, comprises handwritten digits ranging from $0$ to $9$. Widely employed for diverse learning tasks, it is specifically crafted for image classification, serving as a foundational dataset for the development and enhancement of prediction models, especially convolutional deep neural networks. This dataset includes $60,000$ training images and $10,000$ test images, with each digit being centrally positioned in a grayscale image of dimensions $28 \times 28$. In addition to MNIST, we incorporated the CIFAR-10, Fashion-MNIST, and SVHN datasets to evaluate our approach on more diverse and challenging data distributions. The CIFAR-10 dataset contains \(60,000\) \(32 \times 32\) color images across \(10\) classes, with \(50,000\) designated for training and \(10,000\) for testing. The Fashion-MNIST dataset includes \(60,000\) training images and \(10,000\) test images of \(28 \times 28\) grayscale clothing items across \(10\) categories. The SVHN (Street View House Numbers) dataset is a real-world image dataset containing over \(600,000\) \(32 \times 32\) color images of house numbers, with \(73,257\) images in the training set, \(26,032\) in the test set, and \(531,131\) additional images for extra training. For our experiments, all images from the four datasets were converted to grayscale, resized to \(64 \times 64\) pixels, and Gaussian noise was added to the training datasets.

\subsection{Comparative Analysis of Loss Dynamics}

To assess the performance of our proposed BI-DCGAN model, we conducted a comprehensive comparison with the standard DCGAN architecture across four datasets: MNIST, FashionMNIST, CIFAR-10, and SVHN. Figure~\ref{fig:bayesdcgan}~\&~\ref{fig:tdcgan} depict the generator and discriminator loss across different iterations for the MNIST dataset, providing insights into the training process, convergence patterns, and the alignment of our BI-DCGAN model with the conventional DCGAN. Similarly, Figures~\ref{fig:bdcganc}~\&~\ref{fig:tdcganc} present the corresponding results for CIFAR-10, Figures~\ref{fig:bdcganf}~\&~\ref{fig:tdcganf} demonstrate the results for FashionMNIST, and Figures~\ref{fig:bdcgans}~\&~\ref{fig:tdcgans} show the results for SVHN. These comparisons further validate the robustness of our BI-DCGAN approach across datasets of varying complexity.

As it was mentioned, the generator loss measures how well the generator is performing in generating realistic samples, while the discriminator loss assesses the ability of the discriminator to differentiate between real and generated samples. In the initial stages, both losses fluctuate as the model undergoes learning. Subsequently, the generator loss demonstrates a downward trend, reflecting an improvement in the generator's ability to produce realistic samples. The BI-DCGAN exhibits reduced fluctuations in the generator loss throughout the training process compared to the conventional DCGAN. The convergence of both losses indicates that the BI-DCGAN model has reached a certain optimum, suggesting that further improvement is limited. This convergence also signifies that the model has acquired a sufficient level of learning and proficiency in generating and discriminating samples.

\end{document}